\documentclass[letterpaper, 10 pt, conference]{ieeeconf}  % Comment this line out if you need a4paper
\usepackage[english]{babel}
\newtheorem{theorem}{Theorem}

\usepackage{amsmath}
\usepackage{hyperref}
\hypersetup{
    colorlinks=true,
    linkcolor=blue,
    filecolor=magenta,      
    urlcolor=cyan,
    pdftitle={Overleaf Example},
    pdfpagemode=FullScreen,
}
 
\usepackage{bbold}
\usepackage{color,soul}
\usepackage{graphicx}
\usepackage{tikz}
\usepackage{multirow}
\usepackage{subfig}

\usepackage{colortbl}
\DeclareMathOperator{\argmin}{argmin} 
\DeclareMathOperator{\argmax}{argmax} 
\usepackage[ruled,linesnumbered,vlined]{algorithm2e}

\SetCommentSty{mycommfont}
\usepackage{bibentry}
\IEEEoverridecommandlockouts  
\usepackage[skip=2pt,font=scriptsize]{caption}% This command is only needed if 
\newcommand{\thickrelbar}[1][1.5pt]{\ensuremath{\mathrel{\rule[0.5ex]{1em}{#1}}}}

\definecolor{darkpurple}{rgb}{0.4, 0.0, 0.4}

\newcommand{\algname}{{\sc BIGIT*}}
\title{\LARGE \bf
Asymptotically Optimal Sampling-Based Path Planning Using Bidirectional Guidance Heuristic   
}

\author{ Yi Wang and Bingxian Mu
\thanks{%$^{1}$
  }
}

\begin{document}

\maketitle
\thispagestyle{empty}
\pagestyle{empty}

%%%%%%%%%%%%%%%%%%%%%%%%%%%%%%%%%%%%%%%%%%%%%%%%%%%%%%%%%%%%%%%%%%%%%%%%%%%%%%%%
\begin{abstract}
%Motion planning is a popular field in robotics. Graph searches and sampling-based planners are usually employed for motion planning. However, these technologies have limitations. 
%due to their superior performance.that lack large local minima 
This paper introduces Bidirectional Guidance Informed Trees (BIGIT*),~a new asymptotically optimal sampling-based motion planning algorithm. Capitalizing on 
the strengths of \emph{meet-in-the-middle} property in
bidirectional heuristic search with a new lazy strategy, and uniform-cost search, BIGIT* constructs an implicitly bidirectional preliminary motion tree on an implicit random geometric graph (RGG). This efficiently tightens the informed search region, serving as an admissible and accurate bidirectional guidance heuristic. This heuristic is subsequently utilized to guide a bidirectional heuristic search in finding a valid path on the given RGG. 
% Firstly, we apply a bidirectional heuristic search (without collision checks) to  until finding a {$meeting~state$} that satisfies the designed stop condition. Then, \algname employs a uniform-cost search (such as Dijkstra's algorithm) to efficiently update heuristic by establishing bounds that narrow down the informed region where a better valid path may lie. Finally, the provided heuristic is used to conduct a bidirectional search on RGGs. This process reuses the previous search efforts. 
%BIGIT* achieves asymptotic completeness and optimality.  
Experiments show that BIGIT* outperforms the existing informed sampling-based motion planners both in faster finding an initial solution and converging to the optimum on simulated abstract problems in $\mathbb{R}^{16}$. Practical drone flight path planning tasks across a campus also verify our results.

%In this paper, we investigate an asymptotically optimal sampling-based motion planning algorithm -- Bidirectionally Informed Trees (\algname). Unlike asymmetric bidirectional sampling-based planners, such as Adaptively informed trees (AIT*) and Effort Informed trees (ETI*), \algname provides an accurate and probably admissible heuristic to direct the bidirectional search for an incrementally sampling batch of states that is defined as an increasingly dense implicit random geometric graph (RGGs) to find an asymptotically optimal solution. \algname firstly applies a bidirectional search without collision-checks(lazy bidirectional search) to the implicit random geometric graph (RGGs) until finding a state that satisfies the stop condition similar to Meet-in-Middle(MM). Then, \algname employs Dijkstra's algorithm (without collision-checks) with a bounded estimated solution to update the heuristic for a state generated by the lazy bidirectional search from the state of fulfilling the stop condition towards the start and goal, respectively. Finally, the provided heuristic is used to induct a bidirectional search (with collision-checks) for the given RGGs. It shows that \algname is probabilistically complete and asymptotically optimal.We experimentally present the strengths of \algname on two simulated abstract problems in $\mathbb{R}^8$ and $\mathbb{R}^{16}$. The results show \algname outperforms the existing sampling-based motion planners both in faster finding an initial solution and converging to the optima.%
\end{abstract}

%%%%%%%%%%%%%%%%%%%%%%%%%%%%%%%%%%%%%%%%%%%%%%%%%%%%%%%%%%%%%%%%%%%%%%%%%%%%%%%%
\section{INTRODUCTION}
Path planning is a crucial problem in robotics navigation. Its primary objective is to discover a path that efficiently circumvents obstacles, thereby ensuring a collision-free path from a specified start to a designated goal. Over the past decades, path-planning algorithms have been a focal point of extensive research.
%Path planning algorithms, aiming at finding a feasible collision-free path from a start to a goal, are extensively explored in the field of robotics over the past decades. 
Graph-search algorithms, for example, A*~\cite{A*} and Dijkstra's Algorithm~\cite{dijstra}, are capable of finding an optimal path within a given discretized environment. 

A heuristic function estimates the $cost$-$to$-$go$ from any state to a designated goal in a state space. The notable priority of an informed (best-first) search algorithm, like A*, utilizes heuristics to efficiently guide the search towards a goal.  
Identifying optimal solutions can be infeasible or difficult in a complex search environment when the planning time is finite. To tackle this challenge, anytime heuristic planners are devised, like ARA*~\cite{ARA} and AWA*~\cite{AWA}. %Restarting Weighted A* (RWA*)~\cite{RWA}.
These planners rapidly find an initial solution, and then iteratively optimize the solution until the allotted planning time expires. 
Additionally, incremental (replanning) graph-based search algorithms, such as D* Lite ~\cite{Dlite} and LPA*~\cite{koenig2004lifelong}, are efficiently to find a new feasible path by reusing the previous searching efforts in unknown or dynamic domains.

%The important priority of an informed (best-first) search algorithm, such as A*, is to efficiently and quickly find a promising path between any arbitrary start and goal pair by a heuristic function that denotes the estimated $cost$-$to$-$go$ from any state to a preferred goal state in a state space. 

% Admissible heuristics, never overestimating the true $cost$-$to$-$go$ from any state to a goal, are a guarantee for an optimal solution in the informed graph-based search. A* is ensured to find optimal solutions when using such a heuristic. Moreover, if the provided heuristic is monotone/consistent, A* not only returns an optimal solution but does so with optimal efficiency~\cite{A*,Pearl}. That is, A* expands the minimal number of states among all optimal graph-search algorithms with the same consistent heuristic. 

Bidirectional search is a strategy that intends to find an optimal solution by conducting two separate searches simultaneously from both the start and goal ~\cite{bis}. Its key strength lies in the potential to exponentially reduce the number of expanded states, as each search only proceeds to half the depth of the path solution~\cite{kof}. 
%Informative heuristics have significantly improved the search performance for unidirectional searches, and thus it is a natural progression to consider incorporating them into a bidirectional search.
However, coordinating the meeting of the forward and backward searches presents a significant challenge using heuristics ~\cite{bi1,bi2}. {\em Meet-in-the-Middle} (MM)~\cite{mm} guarantees finding an optimal solution, and maintains \emph{meet-in-the middle } priority, such that each expanded state whose \emph{cost-to-come} from the respective search origins never exceeds half of the optimal solution.
 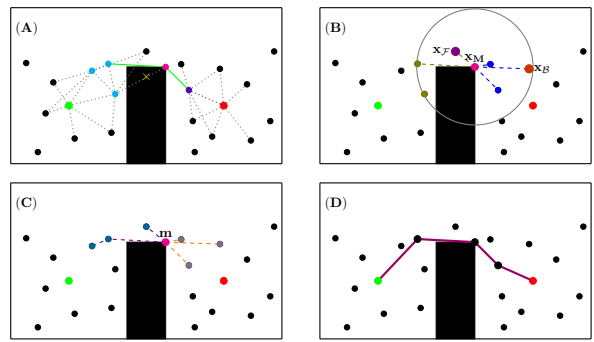
\begin{figure}
    \centering
    %\subfloat[\ref{cPRBMT}]{
      \scalebox{0.515}{
 \begin{tikzpicture} 
\draw[thick] (0,0) -- (7,0) node[anchor=north west]{};
\draw[thick] (0,0) -- (0,4.03) node[anchor=south east]{};
\draw[thick] (0,4.03) -- (7,4.03) node[anchor=south east]{};
\draw[thick] (7,0) -- (7,4.03) node[anchor=south east]{};

\node[anchor=west] at(0,3.5) {$(\bf{A})$};
  \filldraw[black] (3,0) rectangle(4,2.5);
  %%%%%%%%%%%%%%%%%%%%%%%%%%%%%%
 \draw[gray,dotted,thick] (5.5,1.5) -- (4.6,1.9);
  \draw[gray,dotted,thick] (5.5,1.5) -- (4.7,1.1);
   \draw[gray,dotted,thick] (5.5,1.5) -- (5.4,2.45);
  \draw[gray,dotted,thick] (5.5,1.5) -- (5.2,0.7);
   \draw[gray,dotted,thick] (5.5,1.5) -- (6.2,0.6);
  \draw[gray,dotted,thick] (5.5,1.5) -- (6.1,2.05);
%%%%%%%%%%%%%%%%%%%%%%%%%%%%%%%
 \draw[gray,dotted,thick] (1.5,1.5) -- (2.7,1.8);
  \draw[gray,dotted,thick] (1.5,1.5) -- (2.1,2.4);
   \draw[gray,dotted,thick] (1.5,1.5) -- (1.1,2.1);
  \draw[gray,dotted,thick] (1.5,1.5) -- (0.9,1.3);
   \draw[gray,dotted,thick] (1.5,1.5) -- (1.65,0.65);
  \draw[gray,dotted,thick] (1.5,1.5) -- (2.6,0.8);
%%%%%%%%%%%%%%%%%%%%%%%%%%%%%%%
  \draw[gray,dotted,thick] (2.7,1.8) -- (2.52,2.58);
    \draw[green,thick] (4,2.5) -- (2.52,2.58);
  \draw[gray,dotted,thick] (2.7,1.8) -- (4,2.5);
  \draw[gray,dotted,thick] (2.7,1.8) -- (3.5,2.9);
  \draw[gray,dotted,thick] (2.7,1.8) -- (2.1,2.4);
  \draw[gray,dotted,thick] (2.7,1.8) -- (2.6,0.8);
%%%%%%%%%%%%%%%%%%%%%%%%%%%%%
\draw[gray,dotted,thick] (5.5,1.5) -- (4.6,1.9);
\draw[gray,dotted,thick] (4.6,1.9) -- (4.7,1.1);
\draw[gray,dotted,thick] (4.6,1.9) -- (5.4,2.45);
\draw[green,thick] (4.6,1.9) -- (4,2.5);
\draw[gray,dotted,thick] (4.4,2.57) -- (4.6,1.9);
\draw[gray,dotted,thick] (5.2,0.7) -- (4.6,1.9);
%%%%%%%%%%%%%%%%%%%%%%%%%%%%
 \draw[gray,dotted,thick] (2.1,2.4) -- (1.1,2.1);
  \draw[gray,dotted,thick] (2.1,2.4) -- (2.52,2.58);
   \draw[gray,dotted,thick] (2.1,2.4) -- (3.5,2.9);
  \draw[gray,dotted,thick] (2.1,2.4) -- (0.9,1.3);
%%%%%%%%%%%%%%%%%%%%%%%%%%%%
\filldraw[red] (5.5,1.5) circle (2.5pt);
\filldraw[green] (1.5,1.5) circle (2.5pt);
\filldraw[black] (2.6,0.8) circle (2pt);
\filldraw[black] (0.7,0.3) circle (2pt);
\filldraw[black] (1.65,0.65) circle (2pt);
\filldraw[cyan] (2.7,1.8) circle (2pt);
\filldraw[cyan] (2.1,2.4) circle (2pt);
\filldraw[cyan] (2.52,2.58) circle (2pt);
\filldraw[black] (0.4,2.6) circle (2pt);
\filldraw[black] (0.9,1.3) circle (2pt);
\filldraw[black] (1.1,2.1) circle (2pt);
%%%%%%%%%%%%%%%%%%%%%%%%%%%%%
\filldraw[black] (3.5,2.9) circle (2pt);
\filldraw[magenta] (4,2.5) circle (2pt);
\filldraw[black] (4.4,2.57) circle (2pt);
\filldraw[red!30!blue] (4.6,1.9) circle (2pt);
\filldraw[black] (4.7,1.1) circle (2pt);
\filldraw[black] (4.75,0.3) circle (2pt);
\filldraw[black] (5.2,0.7) circle (2pt);
\filldraw[black] (6.2,0.6) circle (2pt);
\filldraw[black] (6.7,1.1) circle (2pt);
\filldraw[black] (5.4,2.45) circle (2pt);
\filldraw[black] (6.8,2.9) circle (2pt);
\filldraw[black] (6.1,2.05) circle (2pt);
\node[anchor=north] at(3.5,2.48) {\textcolor{yellow}{$\times$}};
%%%%%%%%%%%%%%%%%%%%%%%%%%%%%  
\end{tikzpicture}
}
   % }
     %\subfloat[\ref{AEBUH}]{
      \scalebox{0.515}{
 \begin{tikzpicture} 
\draw[thick] (0,0) -- (7,0) node[anchor=north west]{};
\draw[thick] (0,0) -- (0,4.03) node[anchor=south east]{};
\draw[thick] (0,4.03) -- (7,4.03) node[anchor=south east]{};
\draw[thick] (7,0) -- (7,4.03) node[anchor=south east]{};
  \filldraw[black] (3,0) rectangle(4,2.5);
\node[anchor=west] at(0,3.5) {$(\bf{B})$};
%%%%%%%%%%%%%%%%%%%%%%%%%%%%%%

\filldraw[red] (5.5,1.5) circle (2.5pt);
\filldraw[green] (1.5,1.5) circle (2.5pt);
\filldraw[black] (2.6,0.8) circle (2pt);
\filldraw[black] (0.7,0.3) circle (2pt);
\filldraw[black] (1.65,0.65) circle (2pt);
\filldraw[black] (2.7,1.8) circle (2pt);
\filldraw[black] (2.1,2.4) circle (2pt);
\filldraw[black] (2.52,2.58) circle (2pt);
\filldraw[black] (0.4,2.6) circle (2pt);
\filldraw[black] (0.9,1.3) circle (2pt);
\filldraw[black] (1.1,2.1) circle (2pt);
%%%%%%%%%%%%%%%%%%%%%%%%%%%%%
\filldraw[black] (3.5,2.9) circle (2pt);
\filldraw[black] (4.4,2.57) circle (2pt);
\filldraw[black] (4.6,1.9) circle (2pt);
\filldraw[black] (4.7,1.1) circle (2pt);
\filldraw[black] (4.75,0.3) circle (2pt);
\filldraw[black] (5.2,0.7) circle (2pt);
\filldraw[black] (6.2,0.6) circle (2pt);
\filldraw[black] (6.7,1.1) circle (2pt);
\filldraw[black] (5.4,2.45) circle (2pt);
\filldraw[black] (6.8,2.9) circle (2pt);
\filldraw[black] (6.1,2.05) circle (2pt);
%%%%%%%%%%%%%%%%%%%%%%%%%%%%%  
 \draw[red!50!green,thick,dashed] (4,2.5) -- (2.52,2.58);
  \draw[red!50!green,thick,dashed] (4,2.5) -- (3.5,2.9);
    %\draw[gray,thick,dotted] (2.52,2.58) -- (2.7,1.8);
        %\draw[gray,thick,dotted] (2.52,2.58) -- (2.1,2.4);
        \draw[gray,thick] (4,2.5) circle (1.5cm);
 \draw[blue,thick,dashed] (4,2.5) -- (5.4,2.45);
 \draw[blue,thick,dashed] (4,2.5) -- (4.4,2.57);
 \draw[blue,thick,dashed] (4,2.5) -- (4.6,1.9);
 \filldraw[red!50!green] (2.52,2.58) circle (2pt);
\filldraw[red!50!green] (2.7,1.8) circle (2pt);
\filldraw[red!50!blue] (3.5,2.9) circle (3pt);
\filldraw[blue] (4.4,2.57) circle (2pt);
\filldraw[blue] (4.6,1.9) circle (2pt);
\filldraw[red!80!green] (5.4,2.45) circle (3pt);
\node[anchor=south] at(4,2.5) {$\mathbf{\mathbf{x}_M}$};
\node[anchor=east] at(3.5,2.9) {$\mathbf{\mathbf{x}_\mathcal{F}}$};
\node[anchor=west] at(5.4,2.45) {$\mathbf{\mathbf{x}_\mathcal{B}}$};
\filldraw[magenta] (4,2.5) circle (2.5pt);
\end{tikzpicture}
}
    %} 
    %\vspace{-4mm}
        %\subfloat[\ref{AEBUH}]{
      \scalebox{0.515}{
 \begin{tikzpicture}  
\draw[thick] (0,0) -- (7,0) node[anchor=north west]{};
\draw[thick] (0,0) -- (0,4.03) node[anchor=south east]{};
\draw[thick] (0,4.03) -- (7,4.03) node[anchor=south east]{};
\draw[thick] (7,0) -- (7,4.03) node[anchor=south east]{};
  \filldraw[black] (3,0) rectangle(4,2.5);
\node[anchor=west] at(0,3.5) {$(\bf{C})$};
%%%%%%%%%%%%%%%%%%%%%%%%%%%%%%

\filldraw[red] (5.5,1.5) circle (2.5pt);
\filldraw[green] (1.5,1.5) circle (2.5pt);
\filldraw[black] (2.6,0.8) circle (2pt);
\filldraw[black] (0.7,0.3) circle (2pt);
\filldraw[black] (1.65,0.65) circle (2pt);
\filldraw[black] (2.7,1.8) circle (2pt);
\filldraw[black] (2.1,2.4) circle (2pt);
\filldraw[black] (2.52,2.58) circle (2pt);
\filldraw[black] (0.4,2.6) circle (2pt);
\filldraw[black] (0.9,1.3) circle (2pt);
\filldraw[black] (1.1,2.1) circle (2pt);
%%%%%%%%%%%%%%%%%%%%%%%%%%%%%
\filldraw[black] (3.5,2.9) circle (2pt);
\filldraw[black] (4.4,2.57) circle (2pt);
\filldraw[black] (4.6,1.9) circle (2pt);
\filldraw[black] (4.7,1.1) circle (2pt);
\filldraw[black] (4.75,0.3) circle (2pt);
\filldraw[black] (5.2,0.7) circle (2pt);
\filldraw[black] (6.2,0.6) circle (2pt);
\filldraw[black] (6.7,1.1) circle (2pt);
\filldraw[black] (5.4,2.45) circle (2pt);
\filldraw[black] (6.8,2.9) circle (2pt);
\filldraw[black] (6.1,2.05) circle (2pt);
%%%%%%%%%%%%%%%%%%%%%%%%%%%%%  
 \draw[red!50!blue,thick,dashed] (4,2.5) -- (2.52,2.58);
  \draw[red!50!blue,thick,dashed] (2.1,2.4) -- (2.52,2.58);
  \draw[red!50!blue,thick,dashed] (4,2.5) -- (3.5,2.9);
    %\draw[red!50!blue,thick,dashed] (2.52,2.58) -- (2.7,1.8);
 \draw[red!50!yellow,thick,dashed] (4,2.5) -- (5.4,2.45);
 \draw[red!50!yellow,thick,dashed] (4,2.5) -- (4.4,2.57);
     \draw[red!50!yellow,thick,dashed] (4,2.5) -- (4.6,1.9);
 \filldraw[green!40!blue] (2.52,2.58) circle (2pt);
%\filldraw[green!40!blue] (2.7,1.8) circle (2pt);
\filldraw[green!40!blue] (3.5,2.9) circle (2pt);
\filldraw[green!40!blue] (2.1,2.4) circle (2pt);
\filldraw[yellow!40!blue] (4.4,2.57) circle (2pt);
\filldraw[yellow!40!blue] (4.6,1.9) circle (2pt);
\filldraw[yellow!40!blue] (5.4,2.45) circle (2pt);
\filldraw[magenta] (4,2.5) circle (2.5pt);

\node[anchor=south] at(4,2.5) {$\mathbf{m}$};
\end{tikzpicture}
}
   % }
        %\subfloat[\ref{FVP}]{
      \scalebox{0.515}{
 \begin{tikzpicture} 
\draw[thick] (0,0) -- (7,0) node[anchor=north west]{};
\draw[thick] (0,0) -- (0,4.03) node[anchor=south east]{};
\draw[thick] (0,4.03) -- (7,4.03) node[anchor=south east]{};
\draw[thick] (7,0) -- (7,4.03) node[anchor=south east]{};
  \filldraw[black] (3,0) rectangle(4,2.5);

%%%%%%%%%%%%%%%%%%%%%%%%%%%%%%
\node[anchor=west] at(0,3.5) {$(\bf{D})$};
  \draw[red!60!blue,ultra thick] (5.5,1.5) -- (4.6,1.9);
 \draw[red!60!blue,ultra thick] (1.5,1.5) -- (2.52,2.58);
  \draw[red!60!blue,ultra thick] (4,2.5) -- (2.52,2.58);
  \draw[red!60!blue,ultra thick] (4,2.5) -- (4.6,1.9);
%%%%%%%%%%%%%%%%%%%%%%%%%%%%%%
\filldraw[red] (5.5,1.5) circle (2.5pt);
\filldraw[green] (1.5,1.5) circle (2.5pt);
\filldraw[black] (2.6,0.8) circle (2pt);
\filldraw[black] (0.7,0.3) circle (2pt);
\filldraw[black] (1.65,0.65) circle (2pt);
\filldraw[black] (2.7,1.8) circle (2pt);
\filldraw[black] (2.1,2.4) circle (2pt);
\filldraw[black] (2.52,2.58) circle (2.5pt);
\filldraw[black] (0.4,2.6) circle (2pt);
\filldraw[black] (0.9,1.3) circle (2pt);
\filldraw[black] (1.1,2.1) circle (2pt);
%%%%%%%%%%%%%%%%%%%%%%%%%%%%%
\filldraw[black] (3.5,2.9) circle (2pt);
\filldraw[black] (4,2.5) circle (2.5pt);
\filldraw[black] (4.4,2.57) circle (2pt);
\filldraw[black] (4.6,1.9) circle (2.5pt);
\filldraw[black] (4.7,1.1) circle (2pt);
\filldraw[black] (4.75,0.3) circle (2pt);
\filldraw[black] (5.2,0.7) circle (2pt);
\filldraw[black] (6.2,0.6) circle (2pt);
\filldraw[black] (6.7,1.1) circle (2pt);
\filldraw[black] (5.4,2.45) circle (2pt);
\filldraw[black] (6.8,2.9) circle (2pt);
\filldraw[black] (6.1,2.05) circle (2pt);
%%%%%%%%%%%%%%%%%%%%%%%%%%%%%

%%%%%%%%%%%%%%%%%%%%%%%%%%%%%
\end{tikzpicture}
}
    %}
\caption{Demonstrating the Bidirectional Guidance Informed Trees (\algname). (A) Constructing a preliminary bidirectional motion tree via a bidirectional heuristic search with a new lazy strategy. (B) Acquiring estimated bounds for updating heuristic. (C) Updating heuristic with the estimated bounds from (B) using uniform-cost search. (D) Use a bidirectional search with the bidirectional guidance heuristic to find a valid path. 
\label{illustration}
\vspace{-7mm}
}
\end{figure}

 %\begin{figure}[t]
    %\centering 
%\includegraphics[width=84mm,scale=0.9]{f1.jpeg}
    %\caption{\algname is applied for a path planning problem at the UNH-Hover system (PX4).}
   % \label{fig:drone}  
%\end{figure}

However, graph-based search algorithms could incur discretization problems requiring increased computation (curse of dimensionality~\cite{bellman}) in high-dimensional continuous planning domains.

 Sampling-based motion planners, including PRM~\cite{PRM}, and RRTs~\cite{RRTs}, avoid the discretization problems of graph-search by randomly sampling states in a high dimensional continuous search space. RRT* and PRM*~\cite{rrtStar} find optimal solutions as the number of sampled states approaches to infinity. Informed sampling-based motion planning algorithm combining the advantages of anytime heuristic search and lazy search strategy have significantly improved the search performance~\cite{informedrrt,bit}. An accurate heuristic can substantially augment the anytime profile of a sampling-based motion planner, like AIT* and EIT*~\cite{EIT}. Consequently, we use AIT* and EIT* as our baseline. 
 
In this paper, we develop Bidirectional Guidance Informed Trees~(\algname), a novel informed sampling-based planner. \algname\ firstly utilizes a variant of MM algorithm to expeditiously traverse on a given RGG, comprised of a set of informed states, from $X_{Start}$ and $X_{Goal}$ simultaneously. During this search, \algname\ introduces a new lazy strategy: edges do not undergo complete edge collision checks, but an {\em intersected state} triggers sparse collision detection for the first batch samplings. Subsequently, \algname\ leverages uniform-cost search with estimated bounds that narrow down the informed search region to efficiently update the heuristic. Finally, the provided heuristic is used to direct a bidirectional search with collision checks for edges in finding a valid path.
%This leads \algname to outperform the existed sampling-based motion planners, such as EIT* and BIT*.    

\section{BACKGROUND}
 The optimal sampling-based planning problem is defined similarly as in~\cite{rrtStar,bit}.
 %\subsection{Problem Formulation}
  The $n$-dimensional continuous state space of a motion planning problem is represented by $X\subseteq \mathbb{R}^n$, $X_{obs} \subset X$ is defined as the set of all states that collide with fixed obstacles, and $X_{free} = X\setminus X_{obs}$ contains states in collision-free space. The start state is denoted by $\mathit{X_{start}}$ and the goal state by $\mathit{X_{goal}}$, where $\{X_{start},X_{goal}\} \subset X_{free}$. Let $\sigma:[0,1]\mapsto X_{free}$ be a collision-free path from $X_{start}$ to $X_{goal}$, consisting of a sequence of states $\sigma(\tau) \in X_{free}$, where $\forall \tau \in [0,1]$, $\sigma(0) = X_{start}$, and $\sigma(1) = X_{goal}$, and $\sum_{free}$ be the set of all paths from $X_{start}$ to $X_{goal}$. An optimal solution to the problem is to find the path with the minimal cost among  $\sum_{free}$, $\sigma^* = \argmin_{\sigma \in \sum_{free}}\{C(
  \sigma)\}$, where $C$ denotes the path cost of $\sigma$, and $C^*$ is the cost of the optimal path.  
  
  %\subsection{Sampling-Based Motion Planners with Heuristics}
   Heuristics can bias the sampling to reduce the search time and refine the path solution for a sampling-based motion planner. RRT-connect~\cite{rrtconect} promptly finds a feasible path by using the $\mathit{connect}$ $heuristic$ that contiguously tries to connect two trees, one rooted in $X_{start}$ and one in $X_{goal}$. Bidirectional WA* Extend ~\cite{WA_Extend} uses an extend operator to connect the searching frontiers as a greedy heuristic. These algorithms can result in efficiently finding a solution but have no guarantees on optimality. 
   
   In \cite{akgun}, a sampling heuristic is proposed using local biasing to improve the path quality, and better solutions are updated with the node rejection method based on the bidirectional RRT*. Informed RRT*~\cite{informedrrt}, an extension version of RRT*~\cite{rrtStar}, uses an ``ellipsoidal heuristic" to bias the sampling. It accelerates the convergence by increasingly sampling states that could provide a better solution. These planners achieve asymptotic optimality and prompt the convergence rate, but heuristics are not used to direct the search.
   
%\subsection{Lazy Collision Checks and Increasingly Improving Accuracy of Heuristic in Sampling-Based Motion Planning}
 Lazy collision checks on edges accelerate the convergence of a sampling-based motion planner by reducing the cumulatively computational cost.  Lazy-PRM* and Lazy-RRG*~\cite{LazyPRM*}, two anytime sampling-based asymptotically optimal motion planners, reduce the tremendous number of collision checks by finding a better candidate path to the goal firstly. If a path is found, each edge along the path is checked for collision. Collided edges are then removed from the roadmap. However, the heuristics are not employed to guide the search. 

 Incorporating informative heuristics into sampling-based motion planning can significantly improve the search efficiency.
 Batch Informed Trees (BIT*)~\cite{bit} combines the strengths of heuristic search and anytime sampling-based motion planning techniques to quickly find an initial solution using a batch of samples. Subsequently, BIT* continuously refines this solution with new batches of samples that could potentially improve the current solution. This refinement process continues until either further planning becomes unavailable, or the solution converges to an optimal one. Although BIT* uses an admissible heuristic to guide the search on a given RGG, it does not improve the heuristic's accuracy during the search. 
 
 The accuracy of a heuristic contributes to acceleration, and a perfect heuristic makes the search effort trivial.
 AIT* and EIT*~\cite{EIT}, asymmetric bidirectional optimal sampling-based motion planners (ABOPs), provide an adaptive heuristic by updating heuristic using a reverse search from a goal to a start. The accurate heuristic improves the anytime search performance. 
%However, it is important to note that the inadmissible heuristic might be used by EIT* in the forward search to find a solution. This could degrade the search performance (see analysis in.
However, Section~\ref{inadmissible} details how the search efficiency of these planners might be compromised.

MM is the first Bi-HS algorithm that achieives \emph{meet-in-the-middle} property while maintaining optimality. Unfortunately, MM does not achieve competitive running time.

 \subsection{ Potential Disadvantages of ABOPs} \label{inadmissible}
%The accuracy, quality, and informativeness of a heuristic significantly determine the search performance. 

In the domain depicted in Fig.~\ref{fig:enclosewithgap}, a goal is enclosed by a slender wall and a constricted gap. It renders challenges for sampling-based planners with heuristics. Within this setting, both start and goal states are placed in close proximity to the wall. This configuration gives rise to two primary difficulties: (1) it could have numerous invalid edges near the start and goal, and the intervening wall; (2) inaccurate heuristics derive the search into local minima, especially amplified by inadmissible heuristics~\cite{inadmissibleA*}. For ABOPs, if an invalid edge, the best one is selected from the forward edge queue, is detected, the heuristic is updated by the reverse search. This could be computationally intensive, especially in the case with the large number of invalid edges, especially for AIT*. Since EIT* might provide an inadmissible heuristic for the forward search, the search performance might be diminished.  
%lacking large local minima, 

 %AIT* offers no distinct advantage in such a domain, as their reverse search also encounters a similar situation, offering finite guidance for their forward search.
 
 EIT* \cite{EIT} adopts a lazy reverse search that employs adaptive sparse collision checks to provide admissible and inadmissible heuristics to conduct the forward search (Section 3.3.3 in \cite{EIT}). The reverse search begins when the estimated total path cost of the best edge is less than that of the best edge in the forward queue (Alg. 8, line 9- 24.~\cite{EIT}). When the states on the forward edge queue are with inadmissible heuristic values, this could lead to their total estimated path costs being inadmissible. As a result, the processing can incur a longer reverse search time. Moreover, this might lead edges with admissible heuristics on the forward edge queue, even if they are distant from the optimal path on the given RGG, to get prioritized for expansion because they seem ``cheaper" in comparison, requiring more collision checks.

 % As is widely understood, a search using an inadmissible heuristic typically results in 
 % a longer path than one using an admissible heuristic. Similar to BIT*, EIT* places a new batch of samples whose estimated total path costs are limited to the current solution cost. As a consequence, EIT* could potentially return a longer path, leading to a larger informed sampling region. This expansion may cause less efficiency in biasing the sampling and thus renders EIT* waste computational resources by exploring less promising or distant area from the optimal path. Ultimately, this inefficient exploration could hinder the convergence process.
 
 %As a consequence, if the heuristic provided is inadmissible, it could potentially result in a longer path, which can lead to a larger informed sampling region and further reduce the convergence rate.  

%The accuracy of the heuristic is updated as new samples are added by \algname, and a provably admissible heuristic is guaranteed. 

%The admissible heuristic is used for a bidirectional search to find an optimal solution.
 
 %Unlike EIT*, \algname provides an admissible heuristic and updates heuristic as new samples are added.  

 \section{Bidirectional Guidance Informed Trees}
%The high-level overview of \algname is composed of three steps: (1) Using a bidirectional heuristic search (without collision checks) in a given RGG until finding a $meeting~state$; (2) updating heuristic in the bounded informed region for sub-search trees; (3) finding valid paths using admissible heuristics. It is shown that \algname guarantees the asymptotic optimality, see proofs in  section~(\ref{proof}). 

\subsection{Notation}
   $\mathcal{F}$ and $\mathcal{B}$ represent the forward and backward search direction, respectively, while $E$ and $V$ denote the edge and vertex queues. 
   For simplicity, our primarily description concentrates on the forward search, with analogous details for the backward search. 
   Let $\mathbf{x}$ be a single state, $\mathbf{x} \in X_{free}$. $prt_\mathcal{F}(\mathbf{x})$ is the parent state of $\mathbf{x}$, respectively.  
   Let $\mathcal{T_F}:=(V_\mathcal{F},E_\mathcal{F})$ be an explicit forward search tree with a set of states, $\{V_\mathcal{F}\} \subset X_{free}$, and edges, $E_\mathcal{F} = (\mathbf{u}_{\mathcal{F}},\mathbf{v}_{\mathcal{F}}$),  where $\mathbf{u}_\mathcal{F}$ is the source state and  $\mathbf{v}_\mathcal{F}$  is the target state. A {\em priori} heuristic and the lower admissible bound denote the Euclidean Norm through this paper.
   %We note the forward search direction in this section, the backward search direction is analogous.
   
   The function $\widehat{g}_\mathcal{F}(\mathbf{x})$ denotes the admissible estimates of the $cost$-$to$-$come$ from $X_{start}$ to $\mathbf{x}$ through the graph. $g_{\mathcal{F}}(\mathbf{x})$ represents the true $cost$-$to$-$come$  given the current forward search tree.  $\overline{g}_{\mathcal{F}}(\mathbf{x})$ is the lower admissible $cost$-$to$-$come$ bound. $\widehat{c}(\mathbf{u},\mathbf{v})$ and $c(\mathbf{u},\mathbf{v})$ denote the admissible estimate and true edge cost of an edge between states $\{\mathbf{u},\mathbf{v}\} \in X_{free}$ respectively, where $\widehat{c}(\mathbf{u},\mathbf{v}) \leq c(\mathbf{u},\mathbf{v})$. 
   
   $\widehat{h}_\mathcal{F}(\mathbf{x})$ shows the admissible estimates of the $cost$-$to$-$go$ from $\mathbf{x}$ to $X_{goal}$. $h_{\mathcal{F}}(\mathbf{x})$ denotes the true $cost$-$to$-$go$. $\overline{h}_{\mathcal{F}}(\mathbf{x})$ represents the lower admissible $cost$-$to$-$go$ bound. $\widehat{f}_\mathcal{F}(\mathbf{x}):=\widehat{g}_\mathcal{F}(\mathbf{x}) + \widehat{h}_\mathcal{F}(\mathbf{x})$ denotes the admissible estimates of the total path cost from $X_{start}$ to $X_{goal}$ going through $\mathbf{x}$. The admissible estimate gives a subset of states, $X_{\widehat{f}}{ := \{ \mathbf{x} \in X_{free} | \widehat{f}_\mathcal{F}(\mathbf{x}) \leq C_{best} \}}$, that could improve the incumbent best path cost,~$C_{best}$. $\widehat{f}_{min{\mathcal{F}}}$, $\widehat{g}_{min\mathcal{F}}$, and $g_{min\mathcal{F}}$ represent the minimum $\widehat{f}_{\mathcal{F}}$, $\widehat{g}_{\mathcal{F}}$ and $g_\mathcal{F}$, respectively.
   
   Let $U_E$ represent the current true cost of the minimal path solution found to date. As an improved solution is found, these values are updated, continuing upon a specific stop condition is met. 
   For a state that is intersected by a bidirectional heuristic search from both forward and backward directions, we refer to it as an {\em intersecting state}, $\mathbf{x_C}$.
   The set of all intersecting sates is given by $\mathcal{M}$, where $\mathbf{x_M} \in \mathbf{x_C} \subseteq \mathcal{M}$.

      In light of the work in~\cite{mm}, it is readily to show that $\overline{g}_\mathcal{F}(\mathbf{x})\leq \widehat{g}_\mathcal{F}(\mathbf{x}) \leq g_\mathcal{F}(\mathbf{x})\leq \frac{C_{best}}{2}$. It is also noted that $\overline{h}_\mathcal{F}(\mathbf{x})\leq \widehat{h}_\mathcal{F}(\mathbf{x}) \leq h_\mathcal{F}(\mathbf{x})$, $h_\mathcal{F}(\mathbf{x}) = g_\mathcal{B}(\mathbf{x})$, $\widehat{h}_\mathcal{F}(\mathbf{x}) = \widehat{g}_\mathcal{B} (\mathbf{x})$, and $\overline{h}_\mathcal{F}(\mathbf{x}) = \overline{g}_\mathcal{B}(\mathbf{x})$.
      
    %$\widehat{f}_{\mathcal{F}Max}$ and $\widehat{f}_{\mathcal{B}Max}$ represent the upper estimated total path cost bound got from a state $\mathbf{x} \in \mathcal{M}$ , where $\mathbf{x}_\mathcal{M}$ is with the lower $cost$-$to$-$go$ bound in both searching trees, to start and goal states respectively. $\mathit{m}$ denotes the number of states sampled per batch. 

\subsection{Neighbors' Connection}
Similar to BIT*~\cite{bit}, \algname\ incrementally approximates the search space, by adding $m$ informed states in each batch, potentially optimizing the current solution (Alg.~\ref{alg1}, line 9). These informed states can be connected with their neighbors that fall within a radius $r$ (-disc RGG) or that are $k$-nearest neighbors (Alg. \ref{alg1}, line 9) to build edge-implicit RGGs~\cite{informedrrt}. %Note that when $\eta > 3^{n}$, the $k$-nearest neighbors are mutual to each other~\cite{FMT}. 
   %Figure~\ref{illustration} servers as a visual representation, illustrating the workflow of \algname\ throughout the subsequent sections.

\subsection{MM-Variant}\label{cPRBMT}

 Since MM lacks efficiency in running time, we propose a new stopping condition in this paper. Once $U_E \leq (C_{{\mathcal{F},\mathcal{B}}} ~or ~\frac{prmin_\mathcal{F} + prmin_\mathcal{B}}{2})$, \algname terminates the search.

\subsection{Lazy-MM in First Batch Sampling} 
 \algname leverages the lazy-MM, omitting the collision for each edge, to find a candidate path (Alg~\ref{alg1}, Line10). Upon termination, \algname acquires estimated bounds by the neighbors of $\mathbf{x}_C$ and employs Dijkstra algorithm to update the heuristic values of the generated states whose f-value does not exceeds the estimated bounds in both search tree (Alg~\ref{alg1}, line 11 and Alg~\ref{alg2}). Following this, \algname starts to find a valid path. 
   
\subsection{Finding A Valid Path}

  For simplicity, we focus on the operation on the forward search, with analogous operations in the backward search. 
  \algname\ uses the MM-variant to find a valid path. It begins by expanding $X_{start}$ and $X_{goal}$ simultaneously (Alg~\ref{alg1}, line 11-12 and line 15-16). 
  Then \algname\ iteratively selects the best edge $(\mathbf{x},\mathbf{x'})$ from the forward edge queue $Q_\mathcal{F}$. Notably, if the heuristic value of an edge has already been updated, its f-values is refined by dividing the candidate path cost. If the edge is already in the forward motion tree, its target state is directly expanded (Alg~\ref{alg1}, line 22). Otherwise, if $\mathbf{x}$ can improve the current g-value for $\mathbf{x'}$, \algname\ checks its validity (Alg~\ref{alg1}, line 23). If the edge is valid, $\mathbf{x'}$ is added to the forward search tree if not already present (Alg~\ref{alg1}, line 25). Otherwise, the edge ($\mathbf{x},\mathbf{x'}$) is removed from the forward motion tree and the parent is rewired (Alg~\ref{alg1}, line 27-28). Then, $\mathbf{x'}$ is expanded. If the $\mathbf{x'}$ is already in the backward tree and provides a better solution, the incumbent solution $U_E$ is updated (Alg~\ref{alg1}, line 31-32).     

  This process repeats until a valid solution is found, after which a new batch of samples is added (Alg~\ref{alg1}, line 17 and line 35). Notably, the strategy for pruning states aligns with~\cite{bit}. Additionally, we utilize sparse collision detection, as detailed in~\cite{EIT}, for an intersecting state which can provide a better solution in the lazy-MM search.

 \begin{algorithm}
  %\tiny
  \scriptsize
 $U_\mathbf{E} = {\infty}$;\\
 $E_\mathcal{F} = E_\mathcal{B} = \emptyset$;\\
 $X_{samples} \leftarrow{X_{goal} \cup X_{start}}$;\\
     $V_\mathcal{F} \leftarrow{X_{start}}$;  $V_\mathcal{B} \leftarrow{X_{goal}}$;\\
 $Q_\mathcal{F} \leftarrow{\emptyset}$; $Q_\mathcal{B} \leftarrow{ \emptyset}$;\\
 
\Repeat{Stopped}{
  $C_{{\mathcal{F},\mathcal{B}}}:= \mathbf{min}(prmin_\mathcal{F}, prmin_\mathcal{B})$;\\
  \If{\texttt{NeedNewStates}()}{ 
      $X_{samples} \xleftarrow{+}{sample(m,U_\mathbf{E})}$;\\
      \If{\texttt{NotThefirstBatch}}{
          $\texttt{ExpandedVertex}(X_{start})$;\\
    $\texttt{ExpandedVertex}(X_{goal})$;
      }
  }
  \If{$\texttt{$\texttt{TheFirstBatch}~\text{and}~\texttt{LazySearchContinue()}$}$}{
      
  $\texttt{UpdatingHeuristic}(\mathbf{x}_\mathcal{M})$;\\
    $\texttt{ExpandedVertex}(X_{start})$;\\
    $\texttt{ExpandedVertex}(X_{goal})$;
      
  }
  \uElseIf{$U_{\mathbf{E}} > (C_{{\mathcal{F},\mathcal{B}}} ~or ~\frac{prmin_\mathcal{F} + prmin_\mathcal{B}}{2} )$}{
    \eIf{$C_{{\mathcal{F},\mathcal{B}}} = prmin_\mathcal{F}$}{
      
   $(\mathbf{x},\mathbf{x'})  \xleftarrow{-} Q_{\mathcal{F}}$;\\
         
         \uIf{$(\mathbf{x},\mathbf{x'})\in \mathbf{E}_\mathcal{F}$}{   $\texttt{ExpandedVertex}(\mathbf{x'})$;}{
         \uElseIf{$g_\mathcal{F}(\mathbf{x}) + \widehat{c}(\mathbf{x},\mathbf{x'}) < g_\mathcal{F}(\mathbf{x'})$}{
          \If{$\texttt{EdgeValid}(\mathbf{x},\mathbf{x'})$}{
           
              \eIf{$\mathbf{x'} \notin V_\mathcal{F}$}{
                 $V_\mathcal{F} \xleftarrow{+}\mathbf{x'} $;
              }{
                 $E_\mathcal{F} \xleftarrow{-} (prt_\mathcal{F}(\mathbf{x'}), \mathbf{x'})$;
              }
              $prt_\mathcal{F}(\mathbf{x'}) \leftarrow{\mathbf{x}}$,
              $E_\mathcal{F} \xleftarrow{+}(\mathbf{x},\mathbf{x'})$;\\
              $\texttt{ExpandVertex}(\mathbf{x'})$;\\
              $g_\mathcal{F}(\mathbf{x'}) =g_\mathcal{F}(\mathbf{x}) + c(\mathbf{x},\mathbf{x'})$;\\
              \If{$\mathbf{x'} \in V_{\mathcal{B}}$}{
                 $U_\mathbf{E} = \mathbf{min}(U_\mathbf{E},g_\mathcal{F}(\mathbf{x'}) + g_\mathcal{B}(\mathbf{x'}))$;\\
              }
             }

         }
         }
    }{
      // searching in the backward motion tree, analogously;
    }
  }
  \Else{   
     $Prune()$;
     $Q_\mathcal{F} \leftarrow Q_\mathcal{B} \leftarrow \emptyset$;\\
  }
}

\caption{$BIGIT^*(X_{start},X_{goal})$}
\label{alg1}
\end{algorithm}

\setlength{\textfloatsep}{0pt}% Remove \textfloatsep
 \begin{algorithm}
   \scriptsize
// Updating heuristic to start in forward preliminary motion tree;\\
    % \ForEach{$\mathbf{x}_\mathbf{v} \in neighbors(\mathbf{x}_\mathbf{M})$ $and$ $Parent_\mathcal{F}(\mathbf{x}_\mathbf{v}) \neq \emptyset$ }{
    % \eIf{$\mathbf{x_v} \in \mathcal{M}$}{
    %   \If{$\widehat{g}_\mathcal{B}(\mathbf{x_v}) < \frac{U_\mathbf{v}}{2}$ $and$ $\widehat{g}_\mathcal{B}(\mathbf{x_v}) > \widehat{g}_\mathcal{B}(\mathbf{x_M}) $}{
    %             $DQ_\mathcal{B} \xleftarrow{+}\{\mathbf{x}_\mathbf{v}| \widehat{g}_\mathcal{B}(\mathbf{x}_\mathbf{v}) \} $;\\
    %     $\widehat{f}_{\mathcal{B}Max}:= \mathbf{max}(\widehat{f}_{\mathcal{B}Max},\widehat{g}_\mathcal{B}(\mathbf{x}_\mathbf{v}) + \overline{h}_\mathcal{B}(\mathbf{x}_\mathbf{v}))$; 
    %   }
    
    % }{
    % \If{ $\overline{g}(\mathbf{x}_\mathbf{v}) < \frac{U_\mathbf{v}}{2}$ $and$ $\widehat{g}_\mathcal{F}(\mathbf{x}_\mathbf{v}) < \widehat{g}_\mathcal{F}(\mathbf{x}_\mathbf{M})$}{
    %   {
    %       $Parent_{\mathcal{B}}(\mathbf{x}_\mathbf{v}) = \mathbf{x}_\mathbf{u}$;\\
    %     $DQ_\mathcal{B} \xleftarrow{+}\{\mathbf{x}_\mathbf{v}| \widehat{g}_\mathcal{B}(\mathbf{x}_\mathbf{v}) =\widehat{g}_\mathcal{B}(\mathbf{x}_\mathbf{M}) + \widehat{c}(\mathbf{x}_\mathbf{M},\mathbf{x}_\mathbf{v})\} $;\\
    %     ; 
    %   }
    % }

    % }
    % }
 $\widehat{f}_{\mathcal{B}Max}:= \argmax_{\mathbf{x'} \in \texttt{nbrs}(\mathbf{x})~\text{and}~prt_\mathcal{F}(\mathbf{x'}) \neq \emptyset}(\widehat{g}_\mathcal{B}(\mathbf{x'}) + \overline{h}_\mathcal{B}(\mathbf{x'}))$;\\ 

 $\texttt{Dijkstra}(\mathbf{x})$;\\
 //  The operation in the backward tree is analogous;    
% \While{$DQ_\mathbf{B} \neq {\emptyset} $}{
%      $\mathbf{x}_\mathbf{u} \xleftarrow{-}DQ_\mathcal{B}$;\\
%      \If{$\mathbf{x}_\mathbf{u} = X_{start}$}{
%        break;
%      }

%       \ForEach{$\mathbf{x}_\mathbf{v} \in neighbors(\mathbf{x}_\mathbf{u})$ $and$ $Parent_\mathcal{F}(\mathbf{x}_\mathbf{v}) \neq \emptyset$  }{   
%            \If{$\widehat{g}_\mathcal{B}(\mathbf{x}_\mathbf{u}) + \widehat{c}(\mathbf{x}_\mathbf{u},\mathbf{x}_\mathbf{v})+\Bar{h}_{\mathcal{B}}(\mathbf{x}_\mathbf{v}) \leq \widehat{f}_{\mathcal{B}Max}$}{
%              \If{$\widehat{g}_\mathcal{B}(\mathbf{x}_\mathbf{v}) > \widehat{g}_\mathcal{B}(\mathbf{x}_\mathbf{u}) + \widehat{c}(\mathbf{x}_\mathbf{u},\mathbf{x}_\mathbf{v})$}{
%                $Prt_{\mathcal{B}}(\mathbf{x}_\mathbf{v}) = \mathbf{x}_\mathbf{u}$;\\
%                $DQ_\mathcal{B} \xleftarrow{+}\{\mathbf{x}_\mathbf{v}|\widehat{g}_\mathcal{B}(\mathbf{x}_\mathbf{v}) = \widehat{g}_\mathcal{B}(\mathbf{x}_\mathbf{u}) +\widehat{c}(\mathbf{x}_\mathbf{u},\mathbf{x}_\mathbf{v})$;
             
%              }
           
%            }
%       }
    
% }   
\caption{$\texttt{UpdatingHeuristic}(\mathbf{x})$}
\label{alg2}

\end{algorithm}
%\vspace{-5mm}

\setlength{\floatsep}{0pt}% Remove \textfloatsep
 %\vspace{-4mm}

\setlength{\floatsep}{0.1cm}% Remove \textfloatsep
 \begin{algorithm}
  \scriptsize
 %$Q_\mathcal{F} \leftarrow{\emptyset}$; $Q_\mathcal{B} \leftarrow{ \emptyset}$;\\
 // Searching in the forward motion tree;\\
    \ForEach{$\mathbf{x'}\in \texttt{nbrs}(\mathbf{x})$}{
      \If{$\widehat{g}_\mathcal{F}(\mathbf{x'}) > g_\mathcal{F}(\mathbf{x}) + \widehat{c}(\mathbf{x},\mathbf{x'})$}{
         \If{$g_\mathcal{F}(\mathbf{x}) + \widehat{c}(\mathbf{x},\mathbf{x'}) + \widehat{h}_\mathcal{F}(\mathbf{x'}) \leq U_\mathbf{E}~\text{or}~ g_\mathcal{F}(\mathbf{x}) \leq \frac{U_E}{2} $}
         {$Q_\mathcal{F} \xleftarrow{+} \{\mathbf{x},\mathbf{x'}\}$;}
          
      }
    }
//  The searching in the backward tree is analogous;    
\caption{$\texttt{ExpandVertex}(\mathbf{x})$}
\label{alg4}
\end{algorithm}

\section{Analysis}\label{proof}
In this section, we present the proofs for the admissibility of \algname\ ($Theorem~\ref{t1}$), and its asymptotic optimality and probabilistic completeness ($Theorem~\ref{t2}$).  
 Following a similar line in MM~\cite{mm}, \algname\ guarantees $meet~in~the~middle$ when the search stops (Alg.~\ref{alg1}, line 17). Namely, each meeting state is with the lowest $cost$-$to$-$come$ from the forward and backward search trees.
 \begin{theorem}\label{t1}
 \algname\ provides an admissible heuristic on a given RGG.
\end{theorem}
\begin{proof}
We illustrate the case for states in the forward preliminary search tree; an analogous proof covers those in the backward tree.
%This proof is for the states in the forward search tree, the proof for the states in backward search tree is analogous. 
  Let $\mathbf{s}$ be the $meeting~state$ triggering the stop condition, and $\mathbf{s'}$  falling within a radius $r$ or a $k$-nearest neighbor of $\mathbf{s}$, be on the forward search tree. 
     There are two cases to be discussed:
     
    Case1: $\mathbf{s'} \in \mathcal{M}$. Since $\widehat{g}_\mathcal{F}(\mathbf{s'}) + \widehat{g}_\mathcal{B}(\mathbf{s'})$ denotes the minimal total cost found so far, $\widehat{g}_\mathcal{B}(\mathbf{s'})$ is the lowest admissible $cost$-$to$-$come$ from the goal state for $\mathbf{s'}$. 
    
    Case2: $\mathbf{s'} \notin \mathcal{M}$. When $\mathbf{s}$ satisfies the stop condition, $\widehat{g}_\mathcal{B}(\mathbf{s})$ is with the lowest admissible $cost$-$to$-$go$ in backward preliminary motion tree for all generated states neighboring $\mathbf{s}$ in the forward preliminary motion tree. If this were not the case, $\mathbf{s}$ would not fulfill the stop condition. Thus, $\widehat{g}_\mathcal{B}(\mathbf{s'}) = \widehat{g}_\mathcal{B}(\mathbf{s}) + \widehat{c}(\mathbf{s}, \mathbf{s'})$ is the lowest $cost$-$to$-$come$ from the goal state for $\mathbf{s'}$.

    Therefore, when $\mathbf{s'}$ is inserted into the backward vertex queue with $\widehat{g}_\mathcal{B}(\mathbf{s'})$, $\widehat{g}_\mathcal{B}(\mathbf{s'})$ is the lowest estimated $cost$-$to$-$come$ from the goal.
    
    Subsequently, \algname\ iteratively selects a generated state, $\mathbf{s''}$, in the forward preliminary motion tree, for expansion via uniform-cost search strategy. $\mathbf{s''}$ follows $\widehat{f}(\mathbf{s''}) \leq \widehat{f}_{\mathcal{B}Max}$ when it is generated. The search terminates when $X_{Start}$ is selected for expansion. As a result, $\widehat{g}(\mathbf{s''})$ is with the lowest $cost$-$to$-$come$ from the goal state, which implies that $\widehat{g}(\mathbf{s''})$ is admissible ).
      %Then, for a state $\mathbf{t}$ that cannot be updated by Dijkstra's algorithm, \algname sets $\widehat{g}_{\mathcal{B}}(\mathbf{t}) = \widehat{g}_\mathcal{B}(\mathbf{s}) + \widehat{c}(\mathbf{s},\mathbf{t})$. 
\end{proof}

\begin{theorem}\label{t2}
  \algname\ is asymptotically optimal and probabilistically complete.   
\end{theorem}
\begin{proof}
   % In~\cite{bit}, we know that if an admissible heuristic is provided, BIT* asymptotically converges to an optimal solution as the number of the samples, q, goes to infinity, i.e.,
   % \begin{align*}
   %      P(\lim_{q\to\infty}sup C^{BIT^*}_{best,q} = C^*) = 1,
   % \end{align*}
   % where $C^{BIT^*}_{best,q}$ is the current best solution returned by BIT* from q samples.
   The provided admissible heuristic ($Theorem$~\ref{t1}) guides a bidirectional search on a given RGG to find a valid path. The stop condition align with that detailed in MM~\cite{mm}. As a consequence, the search terminates with $C_{best}$. This indicates that \algname\ converges to an optimal solution as $q$ goes to infinity, i.e.,
      \begin{align*}
        P(\lim_{q\to\infty}sup C^{BIGIT^*}_{best,q} = C^*) = 1,
   \end{align*}
 which implies that \algname\ is probabilistically complete. 
\end{proof}
%\begin{itemize}

%\item Use a zero before decimal points: Ò0.25Ó, not Ò.25Ó. Use Òcm3Ó, not ÒccÓ. (bullet list)
%\end{itemize}
%\subsection{Initial Solutions of A Forward, Asymmetric Bidirectional, and Bidirectional Sampling-Based Path Planner } 

%If the non-mutual $k$-nearest neighbors strategy is used to find a valid path by a forward (BIT*), asymmetric bidirectional (AIT*/EIT*), and bidirectional search (\algname) with an admissible heuristic, the initial found solutions can be different. It is known that AIT*/EIT* provides a backward parent and \algname supplies a backward and forward parent. We show the cases as follows:

%Case1: A forward search finds a path from a given RGG. The initial solution found by asymmetric bidirectional or bidirectional search should not be larger than the forward search.

%Case2: A forward search cannot find an initial solution in a given RGG. The asymmetric bidirectional or the bidirectional search can still find an initial solution, leading to a longer initial path length. The right upper figure of Fig.~\ref{unhcm} shows that an initial solution found with non-mutual $k$-nearest neighbors by BIT* (1597.32), EIT* (1596.7), and \algname (1549.01) on UNH Campus Map. The path costs display that \algname returns a better initial solution than BIT* and EIT*.

\section{Experimental Results}

  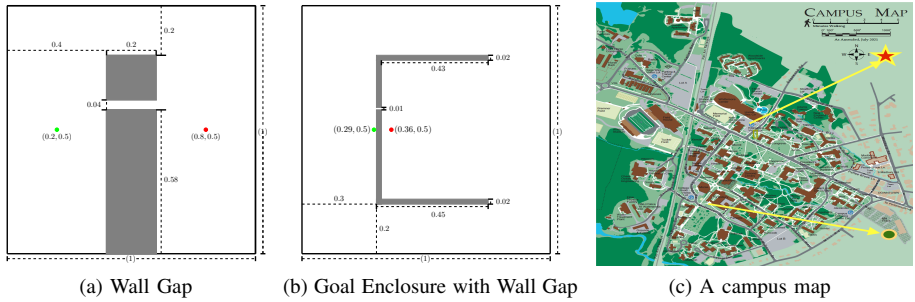
\begin{figure*}[h]
  \vspace{-2mm}
    \centering
    \subfloat[Wall Gap]{
    \scalebox{0.33}{
       \begin{tikzpicture}  
      %\filldraw[green!20] (1,2) -- (6,2) -- (6,0)--(3,0) --cycle; 
           %\filldraw[blue!20] (4,1) -- (6,1/3) -- (6,0)--(4,0) --cycle; 
\draw[ultra thick] (0,0) -- (10,0) node[anchor=north west]{};
\draw[ultra thick] (0,0) -- (0,10) node[anchor=south east]{};
\draw[ultra thick] (0,10) -- (10,10) node[anchor=south east]{};
\draw[ultra thick] (10,0) -- (10,10) node[anchor=south east]{};

\filldraw[gray, ultra thick](4,0) rectangle (6,5.8);
\filldraw[gray, ultra thick](4,6.2) rectangle (6,8);
\filldraw[gray, ultra thick](4,0) rectangle (6,5.8);

\draw[ultra thick] (0,10) -- (10,10) node[anchor=south east]{};
\draw[ultra thick] (10,0) -- (10,10) node[anchor=south east]{};
\draw[ultra thick] (10.2,0) -- (10.4,0) node[anchor=south east]{};
\draw[ultra thick] (10.2,10) -- (10.4,10) node[anchor=south east]{};
\draw[thick,dashed] (10.3,0) -- (10.3,5) node[anchor=south east]{};
\draw[thick,dashed] (10.3,0) -- (10.3,4.8) node[anchor=south east]{};
\draw[thick,dashed] (10.3,5.2) -- (10.3,10) node[anchor=south east]{};
\draw[thick,dashed] (0,-0.2) -- (4.8,-0.2) node[anchor=south east]{};
\draw[thick,dashed] (5.2,-0.20) -- (10,-0.20) node[anchor=south east]{};
\node[anchor=east] at(5.3,-0.2) {\scalebox{0.9}{$(1)$}};
\draw[thick,dashed] (0,8.2) -- (4,8.2) node[anchor=south east]{};
\draw[ultra thick] (4,8.2) -- (4,8) node[anchor=south east]{};
\node[anchor=south] at(2,8.2) {\scalebox{0.9}{$0.4$}};

\draw[thick,dashed] (4,8.2) -- (6,8.2) node[anchor=south east]{};
\draw[ultra thick] (6,8.2) -- (6,8) node[anchor=south east]{};
\node[anchor=south] at(5,8.2) {\scalebox{0.9}{$0.2$}};

\draw[thick,dashed] (6.2,0) -- (6.2,5.8) node[anchor=south east]{};
\draw[ultra thick] (6,5.8) -- (6.4,5.8) node[anchor=south east]{};
\node[anchor=west] at(6.2,3) {\scalebox{0.9}{$0.58$}};

\draw[thick,dashed] (4,5.8) -- (4,6.2) node[anchor=south east]{};
\draw[ultra thick] (4,5.8) -- (3.8,5.8) node[anchor=south east]{};
\draw[ultra thick] (4,6.2) -- (3.8,6.2) node[anchor=south east]{};
\node[anchor=east] at(3.9,6.0) {\scalebox{0.9}{$0.04$}};

\draw[thick,dashed] (6.2,8) -- (6.2,10) node[anchor=south east]{};
\draw[ultra thick] (6,8) -- (6.4,8) node[anchor=south east]{};
\node[anchor=west] at(6.2,9) {\scalebox{0.9}{$0.2$}};

\draw[ultra thick] (0,-0.1) -- (0,-0.3) node[anchor=south east]{};
\draw[ultra thick] (10,-0.1) -- (10,-0.3) node[anchor=south east]{};
\filldraw[green] (2,5) circle(2pt);
\node[anchor=north] at(10.3,5.2) {\scalebox{0.9}{$(1)$}};
\node[anchor=north] at(2,5) {\scalebox{0.9}{$(0.2,0.5)$}};
\node[anchor=north] at(8,5) {\scalebox{0.9}{$(0.8,0.5)$}};
\filldraw[red] (8,5) circle(2pt); 
\end{tikzpicture}   \label{fig:gap}
}
}
    \subfloat[Goal Enclosure with Wall Gap]{
    \scalebox{0.33}{
\begin{tikzpicture}  
   
\draw[ultra thick] (0,0) -- (10,0) node[anchor=north west]{};
\draw[ultra thick] (0,0) -- (0,10) node[anchor=south east]{};
\draw[ultra thick] (0,10) -- (10,10) node[anchor=south east]{};
\draw[ultra thick] (10,0) -- (10,10) node[anchor=south east]{};
\filldraw[gray, thick](3,2) rectangle (7.5,2.2);
\filldraw[gray, thick](3,7.8) rectangle (7.5,8);
\filldraw[gray, thick](3,2) rectangle (3.2,5.8);
\filldraw[gray, thick](3,5.9) rectangle (3.2,7.8);
\filldraw[green] (2.9,5) circle(2pt);

\draw[ultra thick] (10.2,0) -- (10.4,0) node[anchor=south east]{};
\draw[ultra thick] (10.2,10) -- (10.4,10) node[anchor=south east]{};
\draw[thick,dashed] (3.3,5.8) -- (3.3,5.9) node[anchor=south east]{};
\draw[ultra thick] (3.2,5.8) -- (3.4,5.8) node[anchor=south east]{};
\draw[ultra thick] (3.2,5.9) -- (3.4,5.9) node[anchor=south east]{};
\node[anchor=west] at(3.4,5.85) {\scalebox{0.9}{$0.01$}};

\draw[thick,dashed] (0,2) -- (3,2) node[anchor=south east]{};
\draw[ultra thick] (3,2) -- (3,1.8) node[anchor=south east]{};
\draw[thick,dashed] (3,0) -- (3,2) node[anchor=south east]{};
\node[anchor=south] at(1.5,2) {\scalebox{0.9}{$0.3$}};
\node[anchor=west] at(3,1) {\scalebox{0.9}{$0.2$}};
\draw[thick,dashed] (7.5,1.9) -- (3,1.9) node[anchor=south east]{};
\draw[ultra thick] (7.5,2) -- (7.5,1.8) node[anchor=south east]{};
\node[anchor=north] at(5.5,1.8) {\scalebox{0.9}{$0.45$}};

\draw[thick,dashed] (3.2,7.7) -- (7.5,7.7) node[anchor=south east]{};
\draw[ultra thick] (7.5,7.8) -- (7.5,7.6) node[anchor=south east]{};
\draw[ultra thick] (3.2,7.8) -- (3.2,7.6) node[anchor=south east]{};
\node[anchor=north] at(5.5,7.7) {\scalebox{0.9}{$0.43$}};

\draw[thick,dashed] (7.6,2) -- (7.6,2.2) node[anchor=south east]{};
\draw[ultra thick] (7.5,2) -- (7.7,2) node[anchor=south east]{};
\draw[ultra thick] (7.5,2.2) -- (7.7,2.2) node[anchor=south east]{};
\node[anchor=west] at(7.7,2.1) {\scalebox{0.9}{$0.02$}};

\draw[thick,dashed] (7.6,7.8) -- (7.6,8) node[anchor=south east]{};
\draw[ultra thick] (7.5,7.8) -- (7.7,7.8) node[anchor=south east]{};
\draw[ultra thick] (7.5,8) -- (7.7,8) node[anchor=south east]{};
\node[anchor=west] at(7.7,7.9) {\scalebox{0.9}{$0.02$}};

\draw[thick,dashed] (5.2,-0.20) -- (10,-0.20) node[anchor=south east]{};
\draw[thick,dashed] (0,-0.2) -- (4.8,-0.2) node[anchor=south east]{};
\node[anchor=east] at(5.3,-0.2) {\scalebox{0.9}{$(1)$}};
\draw[thick,dashed] (10.3,0) -- (10.3,4.8) node[anchor=south east]{};
\draw[thick,dashed] (10.3,5.2) -- (10.3,10) node[anchor=south east]{};
\node[anchor=north] at(10.3,5.2) {\scalebox{0.9}{$(1)$}};
\node[anchor=east] at(2.9,5) {$(0.29,0.5)$};
\node[anchor=west] at(3.6,5) {$(0.36,0.5)$};
\filldraw[red] (3.6,5) circle(2pt);
\draw[ultra thick] (0,-0.1) -- (0,-0.3) node[anchor=south east]{};
\draw[ultra thick] (10,-0.1) -- (10,-0.3) node[anchor=south east]{};
\end{tikzpicture} \label{fig:enclosewithgap}
}
} 
       \subfloat[A campus map]{
    \scalebox{1}{    \includegraphics[height=3.5cm,width=42mm,scale=1]{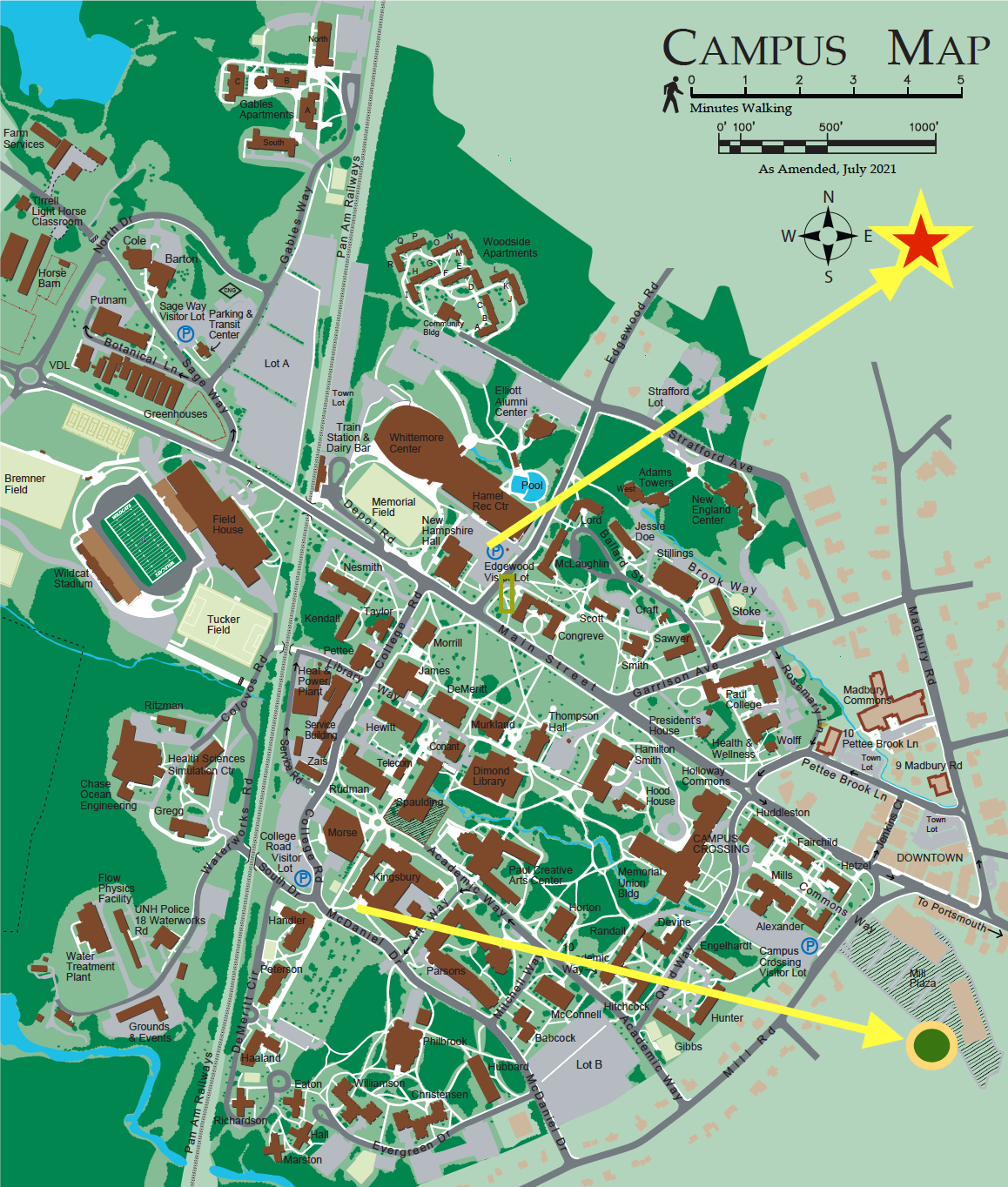}}\label{unhcm} }

\caption{  Experimental environments: abstract problems in (a) and (b); real-world scenarios in (c).} 
    \label{fig:TestDomain}
\end{figure*}

 \begin{figure*}
    \centering
    \subfloat[Wall Gap ($\mathbb{R}^{16}$)]{
        \scalebox{0.32}{\includegraphics[width= \textwidth]{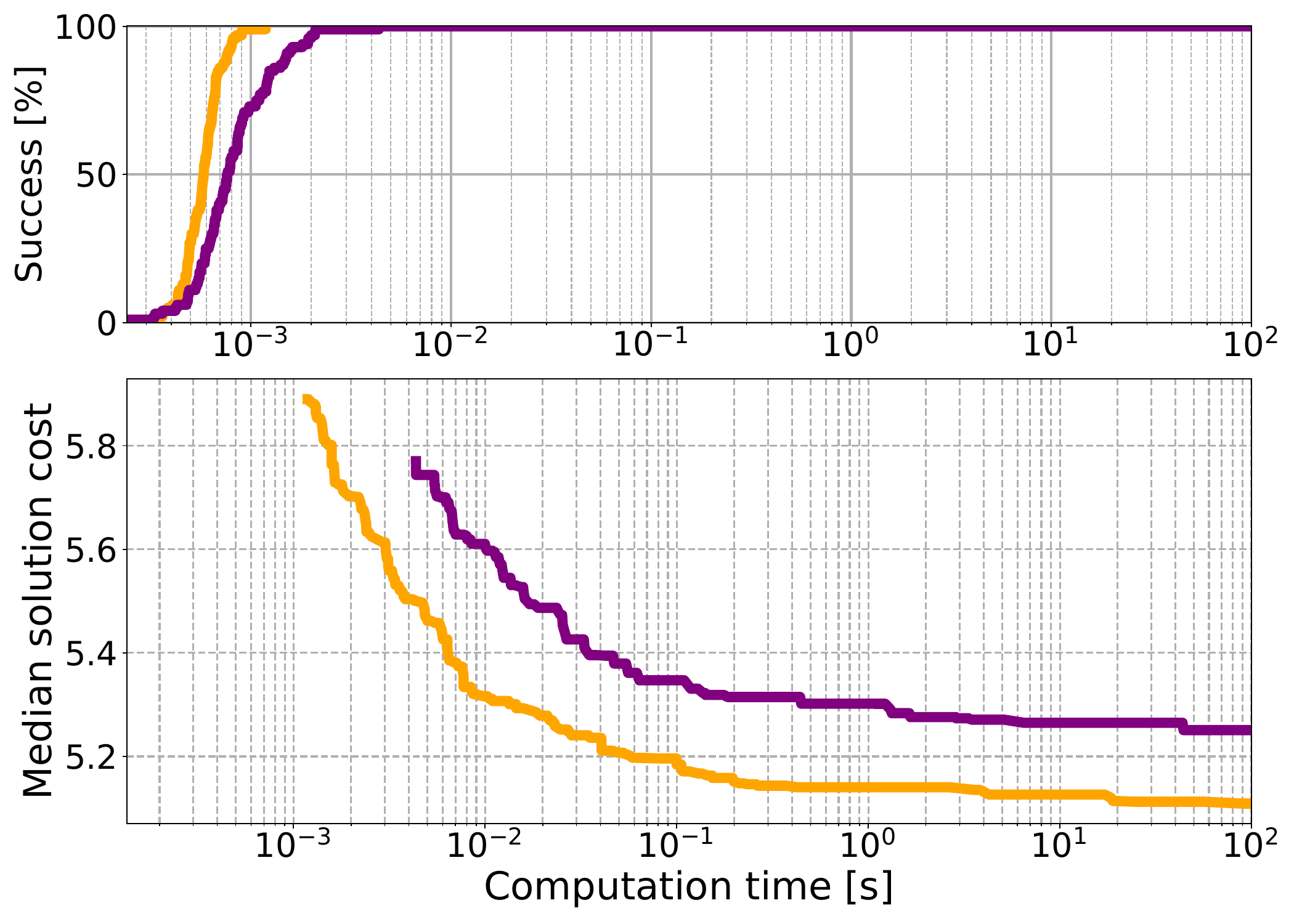} \label{fig:wall} }
    }
        \subfloat[Goal Enclosure with Wall Gap ($\mathbb{R}^{16}$)]{
        \scalebox{0.32}{\includegraphics[width= \textwidth]{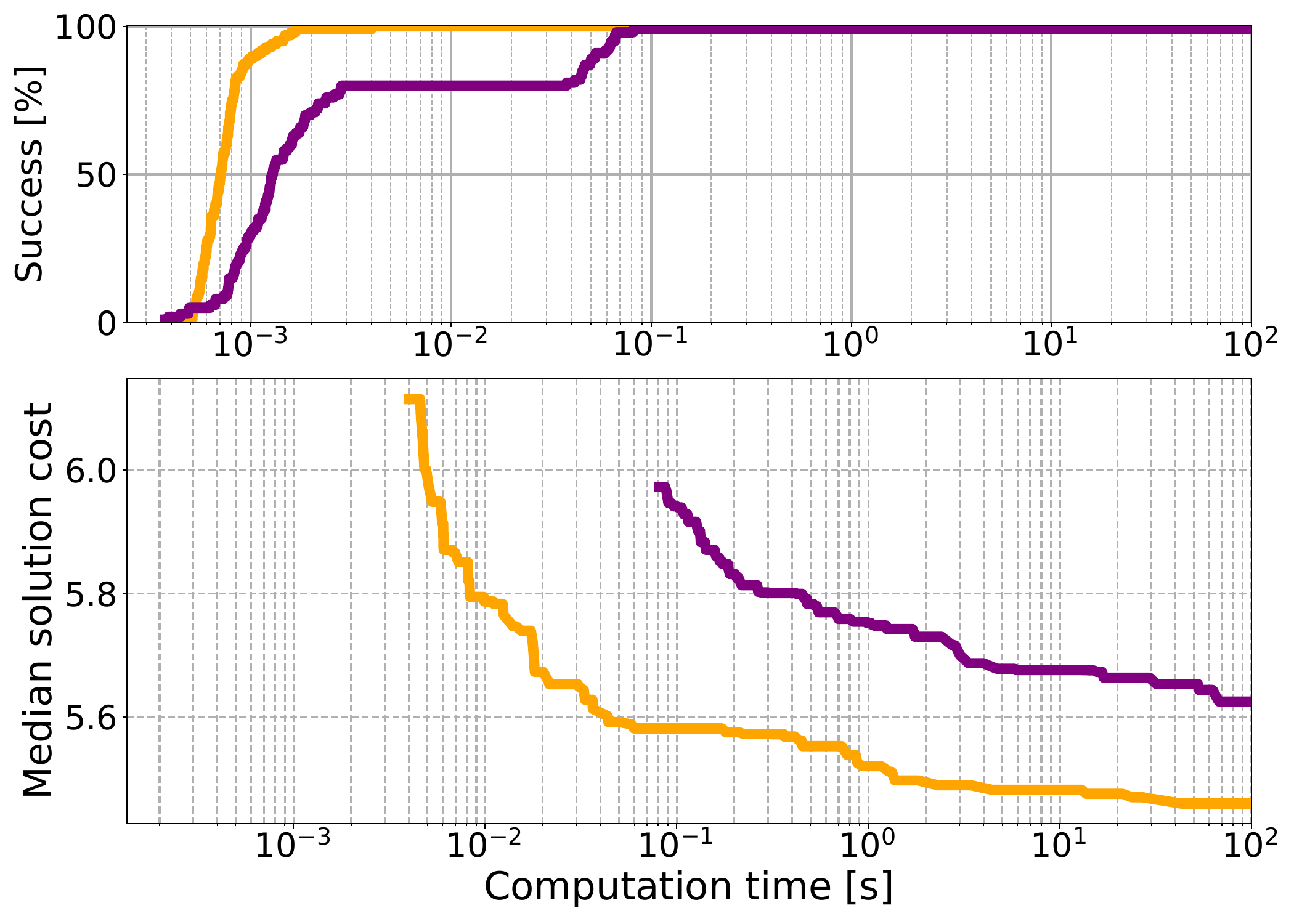} \label{fig:ex2} }
    }
    \subfloat[Campus Map ($SE(2)$)]{
        \scalebox{0.32}{\includegraphics[width= \textwidth]{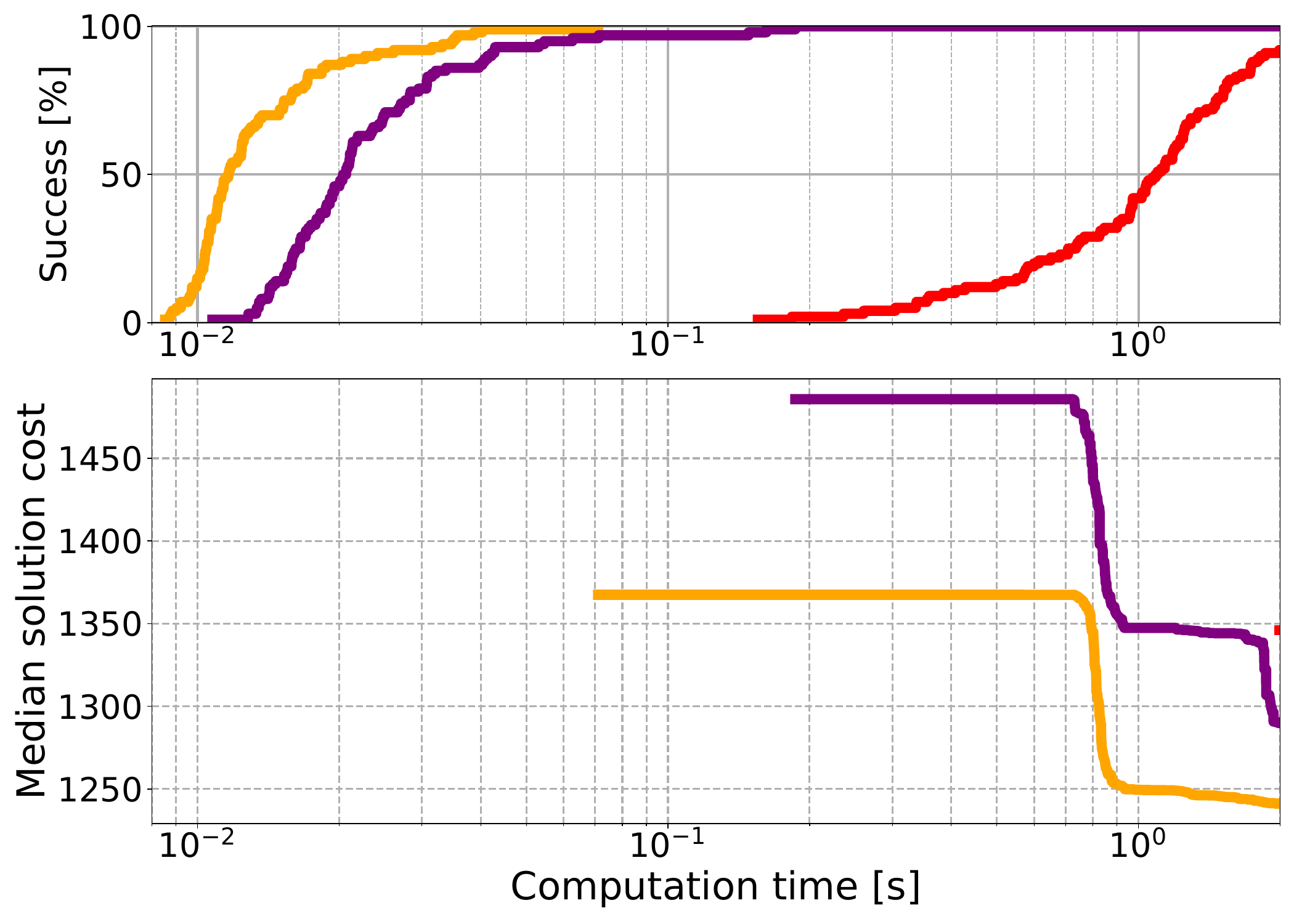} \label{unhat}}
    }
    % \subfloat[Wall Gap In $\mathbb{R}^8$]{
    % %\scalebox{0.9}{
    %   \includegraphics[width=160mm,scale=1]{NewPaper.pdf}}
    %}  % start with (b)
  
%     \subfloat[Wall Gap In $\mathbb{R}^{16}$]{
%     %\scalebox{0.63}{
% \includegraphics[width=160mm,scale=1]{R16WALL.png}}
    %}
    \caption{Comparative performance of BIGIT* and EIT* across different domains in terms of median solution cost and success rate over time, with 95$\%$ confidence interval in each plot. \textcolor{orange}{$\thickrelbar$} and \textcolor{darkpurple}{$\thickrelbar$} denote \algname and EIT*, respectively. }
     
\end{figure*}

%  \begin{figure}[t]
%     \centering 
%     %\subfloat[]{}
%     \scalebox{0.9}{
%     \includegraphics[width=90mm,scale=0.8]{UNH.pdf}}

%     \caption{Initial and final search performances on selected Campus Map. We set cut = 0 to restrict the violin range within the range of our testing data. The violin plots display the 95$\%$ confidence interval (lines that extend from the center). The results show \algname\ outperforms EIT* in faster finding an initial solution and in prompting the convergence.}
%     \label{fig:VPUNH}  
% \end{figure}

We implement \algname\ in Open Motion Planning Library (OMPL)~\cite{sucan2012open} and compare its performance with other informed sampling-based motion planners in OMPL, specifically EIT* and AIT*. The evaluation is conducted on abstract problems in $\mathbb{R}^{16}$, as well as on a Campus Map. All planners use the neighbors-connection strategy with $\eta = 1.001$ and 200 line segments for collision checks between any two states. Notably, solution quality remained consistent even when increasing the number of line-segment checks to 10,000. Notably, $k$-NN is used for two abstract problems while $r$-disc for campus map. 

% In addition to anytime plots representing mean cost over time, we utilize the violin plots for a more detailed distribution perspective, particularly concerning initial path costs, search times, and final path costs. Within these violin plots, the white dots serve as indicators of the median values for initial search time and path cost. This is motivated by the initial search time and mean cost in the anytime plots determined at the maximum first solution time. %Leveraging both plot types ensures a comprehensive representation of each planner's performance, especially concerning initial path costs, search times, and final path costs. 

All planners are run on a desktop with 16GB of RAM and Intel I7-7700k CPU running Ubuntu 18.04 and are implemented in C++. 
\subsection{Planning for Abstract Problems}

Two-dimensional test domains illustrations, in Fig~\ref{fig:gap} and Fig~\ref{fig:enclosewithgap}. Start and goal states are depicted as green (\textcolor{green}{$\bullet$}) and red  (\textcolor{red}{${\bullet}$}), respectively. Each state space is limited to the interval [0,1] for sampling, while gray areas indicate the collision regions.
 All tested planners sample 100 states per batch and  Each planner is run 100 times with different random seeds, and is offered 100 seconds in $\mathbb{R}^{16}$, respectively.  

 The planners are evaluated on the same domain as EIT*~\cite{EIT}, Wall Gap featuring a wall with a narrow passage in Fig.~\ref{fig:gap}. As EIT* has already outperformed others here, our main comparison involves EIT* and 
 \algname . As showcased in Fig.~\ref{fig:wall}, \algname\ outperforms EIT* in the evaluated key metrics. 

Another tested domain, depicted in Fig.~\ref{fig:enclosewithgap}, is adapted from Fig.~3b~\cite{ait}. Comparative results of all evaluated planners are presented in Fig.~\ref{fig:ex2}. In $\mathbb{R}^{16}$, \algname\ stands out as the only planner to reach the $100\%$ success rate, with EIT, and AIT* achieving rates of $99\%$, and $16\%$, respectively.  As demonstrated in~\ref{fig:ex2}, \algname\ outperforms EIT* in defined key metrics.

\subsection{Planning on A Campus Map}
While path planning in 2-dimensional terrain might seem simpler than in higher dimensions, real-world 2D settings pose distinct challenges due to their intricate and diverse obstacle configurations, especially for applications like fixed-altitude drone navigation and ground vehicles, where precision in pathfinding is paramount. All evaluated planners are applied for a path planning task on a Campus Map sized at 3160 $\times$ 3724, which spans approximately $4.7659\times10^6$ $\text{m}^2$. The planning mission starts, depicted as green (\textcolor{green}{$\bullet$}), and terminates at red  (\textcolor{red}{${\bullet}$}), respectively, as illustrated in Fig.~\ref{unhcm}. Dark brown, the shadow of the building, and dark green areas are identified as obstacles. Each planner samples 6000 states per batch across this terrain. All planners are run 100 trails and given 2 seconds for each trail. The quadrotor, steered by a PX5 flight controller, traces the planned waypoints. The dimension of the drone is $50\times50\times21.5$ $\text{cm}^3$ with a GPS accuracy 0.6 m. Its size is considered for collision checks. %In this experiment, the quadrotor is treated as a disc robot.

The experimental results for all planners are demonstrated in
Fig.~\ref{unhat}.
Within this domain, only AIT* attains a $96\%$ success rate while others achieve $100\%$. Notably, AIT* does not show particular advantages over other planners in this domain. \algname\ consistently outperforms others in quickly converging to the optimum.
%An experimental test video is found at  \textcolor{blue} {}.

 \section{Conclusions and Future Work}
  In this paper, we present \algname , a novel asymptotically optimal sampling-based planner. \algname\ exploits the strengths of bidirectional heuristic search with a new lazy strategy, and the uniform-cost search adhering to estimated bounds, to construct a bidirectional preliminary motion tree. This tree can efficiently refine the informed region for searching more efficiently and serves an accurate and admissible heuristic. This heuristic is then used to direct a bidirectional search to find a valid path. Empirical evidence from the simulated and real-world experiments highlights that \algname\ outperforms other informed sampling-based planners in finding initial solutions and optimizing convergence in the tested domains. Acknowledging the challenges presented by Kinodynamic motion planning, due to following system constraints, the future work aims to leverage bidirectional geometric search efforts to guide forward searches within kinodynamic constraints.

%\section*{Acknowledgement}
%We would like to  express our sincere gratitude to Prof. Wheeler Ruml for his comments on drafts of this work, and Debarpan Bhowmick for his help with experiments on Pixhawk 5X drone flight tests. We also want to thank the anonymous reviewers for their valuable feedback and suggestions.
      
%\addtolength{\textheight}{-12cm}   

%%%%%%%%%%%%%%%%%%%%%%%%%%%%%%%%%%%%%%%%%%%%%%%%%%%%%%%%%%%%%%%%%%%%%%%%%%%%%%%%

%%%%%%%%%%%%%%%%%%%%%%%%%%%%%%%%%%%%%%%%%%%%%%%%%%%%%%%%%%%%%%%%%%%%%%%%%%%%%%%%

%%%%%%%%%%%%%%%%%%%%%%%%%%%%%%%%%%%%%%%%%%%%%%%%%%%%%%%%%%%%%%%%%%%%%%%%%%%%%%%%
%\section{Discussion}
   %In the real world, the searching environments are not always ideally sparse, which could make \algname\ fails, so we are planning to solve this issue in the future work va.

%\section*{ACKNOWLEDGE}

%%%%%%%%%%%%%%%%%%%%%%%%%%%%%%%%%%%%%%%%%%%%%%%%%%%%%%%%%%%%%%%%%%%%%%%%%%%%%%%%
\bibliographystyle{IEEEtran}
\bibliography{ICRA}

\end{document}